\documentclass[conference]{IEEEtran}
\usepackage{fancyhdr}
\usepackage[utf8]{inputenc}
\usepackage{amsmath}
\usepackage{amsthm}
\usepackage{amssymb}
\newtheoremstyle{mystyle}
  {}
  {}
  {\itshape}
  {}
  {\bfseries}
  {.}
  { }
  {}
\theoremstyle{mystyle}
\newtheorem{theorem}{Theorem}
\newtheorem{definition}{Definition}
\newtheorem{assumption}{Assumption}

\usepackage{relsize}
\usepackage[noend]{algpseudocode}
\usepackage[Algorithm]{algorithm}
\usepackage{mathtools}
\usepackage{booktabs}
\usepackage{subcaption}
\usepackage{ upgreek }

\begin{document}

\graphicspath{{./Figures/}}

\title{FedGen: Generalizable Federated Learning for Sequential Data}
\author{}
\author{\IEEEauthorblockN{Praveen Venkateswaran}
\IEEEauthorblockA{\textit{IBM Research} \\
praveen.venkateswaran@ibm.com}
\and
\IEEEauthorblockN{Vatche Isahagian}
\IEEEauthorblockA{\textit{IBM Research} \\
vatchei@ibm.com}
\and
\IEEEauthorblockN{Vinod Muthusamy}
\IEEEauthorblockA{\textit{IBM Research} \\
vmuthus@us.ibm.com}
\and
\IEEEauthorblockN{Nalini Venkatasubramanian}
\IEEEauthorblockA{\textit{UC Irvine} \\
nalini@uci.edu}
}

\maketitle

\begin{abstract}

Existing federated learning models that follow the standard risk minimization paradigm of machine learning often fail to generalize in the presence of spurious correlations in the training data. In many real-world distributed settings, spurious correlations exist due to biases and data sampling issues on distributed devices or clients that can erroneously influence models. Current generalization approaches are designed for centralized training and attempt to identify features that have an invariant causal relationship with the target, thereby reducing the effect of spurious features. However, such invariant risk minimization approaches rely on apriori knowledge of training data distributions which is hard to obtain in many applications. In this work, we present a generalizable federated learning framework called FedGen, which allows clients to identify and distinguish between spurious and invariant features in a collaborative manner without prior knowledge of training distributions. We evaluate our approach on real-world datasets from different domains and show that FedGen results in models that achieve significantly better generalization and can outperform the accuracy of current federated learning approaches by over $24\%$.
\end{abstract}

\section{Introduction}\label{sec:introduction}

The rise in popularity of edge devices (e.g., mobile phones, wearables, IoT sensors) has resulted in the generation of massive amounts of data, resulting in the need for distributed model training over a large number of such devices for many machine learning applications. 
Traditional cloud-based ML approaches often require the data to be moved to central cloud servers \cite{chilimbi2014project, dean2012large}, which further exacerbates network bandwidth demands, data privacy concerns, and other security issues \cite{sharma2021image}. 
Federated Learning (FL) using the cloud has emerged as an attractive paradigm \cite{xu2020fedmax, khan2021federated}, allowing local devices or clients to collaboratively train a shared global model without any data sharing. The typical federated learning paradigm involves two stages - (i) clients train models with their local datasets independently, and (ii) a central server gathers the locally trained models and aggregates them to obtain a shared global model. 

A standard approach for federated learning has been FedAvg \cite{mcmahan2017communication} where parameters of local models are averaged element-wise with weights proportional to sizes of the client datasets. However, FedAvg has been shown to have several shortcomings \cite{wang2020federated,kairouz2019advances}, in particular its performance when the data across clients are non-iid. While several improvements have been proposed to handle non-iid data \cite{li2018federated,mohri2019agnostic}, they rely on empirical risk minimization (ERM) techniques which has been shown by numerous examples to be brittle in generalizing to out-of-distribution (OOD) data. 

This is partly due to being influenced by spurious correlations and data biases that fail to hold outside training data distributions \cite{tzeng2017adversarial,recht2019imagenet}. A classic example was highlighted by \cite{beery2018recognition} where a model, trained to classify images of cows in pastures and camels in the desert, failed when the backgrounds were switched because it was influenced by the \textit{spurious correlation} (i.e., green pastures with cows and sandy deserts with camels) rather than relying on the \textit{invariant features} (i.e., the cows and camels themselves). 
In another example, \cite{degrave2021ai} showed how models trained to detect COVID-19 from chest radiographs relied on spurious features (e.g. hospital metadata) rather than medical pathology features, and failed when tested in new hospitals. 


\begin{figure}[!t]
\centering
\begin{minipage}[b]{.49\linewidth}
\centering
\includegraphics[width=\textwidth]{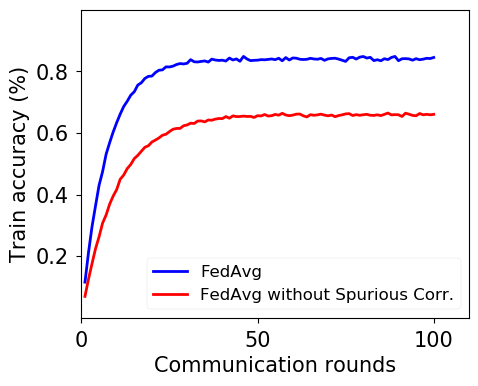}
\subcaption{Training Accuracy}
\label{fig:intro_train_accuracy}
\end{minipage}
\begin{minipage}[b]{.495\linewidth}
\centering
\includegraphics[width=\textwidth]{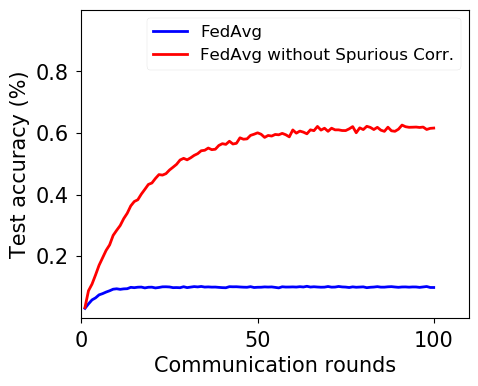}
\subcaption{Test Accuracy}
\label{fig:intro_test_accuracy}
\end{minipage}
\caption{Comparing the generalizability of FedAvg with and without spurious correlations in the training data}
\label{fig:intro_example}
\end{figure}

To demonstrate the sensitivity of FedAvg and associated empirical risk minimization approaches to spurious correlations in the training data, we use the Punctuated SST-2 dataset \cite{choe2020empirical}. 
It consists of sentences and their binary sentiment labels divided into two training distributions. 
A punctuation mark, either a '!' or '.', is introduced as a spurious feature with an $80\%$ and $90\%$ spurious correlation with each of the binary sentiment labels in the two training distributions respectively, and only has a $10\%$ correlation in the test distribution (i.e.) out-of-distribution (OOD) setting. 
Hence, any model relying on this punctuation feature rather than the sentence while predicting the sentiment would do well during training, but perform poorly on the test set where the spurious correlation no longer holds. 

We distribute the training data over ten clients and compare the performance of FedAvg on data with the spurious feature against its performance if the spurious feature was not present. Figure \ref{fig:intro_example} shows the training accuracy and the resulting out-of-distribution test accuracy of FedAvg in the two cases. We observe that when the spurious feature is present, FedAvg is influenced significantly by it, thereby resulting in a higher training accuracy than if it were not present. However, this influence results in poor generalization, and it demonstrates very low test accuracy on the out-of-distribution (OOD) data. The poor generalizability of FedAvg can be clearly seen from the performance of the two cases in Figure \ref{fig:intro_example}, since the only difference in the datasets is the presence of a spurious feature. 

The example above highlights the drawbacks of existing federated learning approaches. In this work, we address the drawbacks by developing a generalizable federated learning approach called FedGen. Our approach develops a distributed masking function that enables clients to collaboratively identify spurious features or correlations in the data thereby ensuring that the resulting aggregated model relies only on the invariant features. 
Prior work on generalization has predominantly focused on centralized training settings for computer vision tasks. However, given the increased use of time-series and sequential data for distributed edge and cloud applications, where spurious correlations can often be present, it serves as the focus of this work. 

Our contributions can be summarized as follows:
\begin{itemize}
    \item We present a generalizable federated learning framework (FedGen) that can handle the presence of spurious correlations or biases in the training data.
    \item FedGen uses a distributed masking function that allows the clients to collaboratively identify and distinguish between spurious and invariant features during the training process even as each client trains on only its local data. 
    \item We formally prove the correctness and present results on the convergence of FedGen. 
    \item We compare our framework with standard centralized training as well as existing federated learning approaches on a variety of sequential and time-series datasets from real-world domains and demonstrate the significantly improved generalizability of aggregated models using FedGen.
\end{itemize}


\section{Background}\label{sec:background}

Consider a set of $K$ devices or clients, each with data samples drawn from $D_k \in \mathcal{D}$ distributions on $X \times Y$, where $X$ is the set of input features and $Y$ is the target variable.
We denote $n_k$ as the number of data samples on device $k \in K$, and $N = \sum_{k=1}^K n_k$ as the total number of samples across the $K$ devices.

The general federated learning problem involves learning a single global statistical model from these $K$ devices under the constraint that device-generated data are stored and processed locally, with only intermediate model parameter updates being communicated periodically with a central server, which aggregates them. The goal is typically to minimize the following objective function:
\begin{equation}
    \min_w F(w) \;, \; \text{where} \; F(w) = \sum_{k=1}^K \frac{n_k}{N} f_k(w_k)
\end{equation}
where $f_k$ is the local objective function for the $k$th device. The local objective function is often defined by existing approaches using Empirical Risk Minimization (ERM) which minimizes the average training loss in a distribution agnostic manner:
\begin{equation}\label{eq:federated_local_loss}
    f_k(w_k) = \mathbb{E}_{(x_k,y_k) \sim {\mathcal{D}}}[\ell(x_k,y_k);w_k]
\end{equation}

where $x_k \in X$, $y_k \in Y$, and $\ell$ is a given loss function.
One of the popular aggregation methods is Federated Averaging (FedAvg) \cite{mcmahan2017communication} which uses ERM and averages the local model parameters of the devices element-wise during aggregation as summarized in Algorithm \ref{alg:fedavg}. While ERM has been shown to work well in practice for i.i.d. data \cite{vapnik1992principles}, it has been shown to fail when (i) training distributions on the devices are non-iid and leading each device to optimize its local objective as opposed to the central objective -- potentially hurting convergence or even causing divergence of the aggregated model \cite{bonawitz2019towards}, and (ii) test distributions differ significantly from training distributions which is a natural consequence when spurious correlations are present in the training data \cite{tzeng2017adversarial}.

Several alternative approaches to FedAvg have been proposed to address the first issue. FedProx \cite{li2018federated} limits the impact of local updates by keeping them close to the global model where they add a proximal term to the local client objective functions $f_k$:
\begin{equation}\label{eq:fedprox}
    f_k(w_k) + \frac{\mu}{2} || w_k - w^t ||^2
\end{equation}
where the proximal term penalizes large deviations of the client model parameters, $\mu$ is a hyper-parameter and we see that when $\mu = 0$, this reduces to FedAvg. 
Agnostic Federated Learning (AFL) \cite{mohri2019agnostic} is another variant of FedAvg, which optimizes a centralized distribution that is a mixture of client distributions $D_k$ by modeling it as $D_\lambda = \sum \lambda_k D_k$, where $\lambda$ is the mixture weight. Their goal is to find a solution that works for any $\lambda$. However, these approaches assume that all correlations present in the training distributions are invariant, and therefore cannot handle spurious correlations resulting in aggregated models that are not generalizable as we will demonstrate. 

Finding the set of invariant features has emerged as a popular approach for generalization in centralized training settings. Invariant Risk Minimization (IRM) \cite{arjovsky2019invariant} tries to find an invariant data representation $\Phi(X)$ and assumes that the training data comes from multiple distributions, and that features whose distributions vary across the training data are likely to also vary between the training and test datasets and hence should be treated as spurious correlations. They define the following loss or risk function:
\begin{equation}
    \min_\Phi \sum_{D \in \mathcal{D}} \mathcal{L}_{D}(\Phi(X_D)) + \lambda || \nabla_w \mathcal{L}_D(w^\top \Phi(X_D)) ||^2_2
\end{equation}
where the second term serves as a regularizer that balances between predictive power within a distribution (ERM), and the invariance of the predictor across distributions. 
There have been several extensions to the IRM framework.  For instance \cite{ahuja2020invariant} propose a game theoretic approach to IRM, while \cite{krueger2021out} introduce the notion of risk extrapolation to encourage strict equality between training risks.

While the goal of identifying the invariant feature set has been shown to be effective for out-of-distribution generalization over a variety of datasets and serves as a base for our approach, IRM based approaches have been shown to have inherent drawbacks \cite{nagarajan2020understanding, rosenfeld2020risks}. Of particular importance is its reliance on prior knowledge of training data distributions, absent which, its performance is similar to ERM \cite{venkateswaran2021environment}. This knowledge is especially hard to obtain in federated settings, where the local data on each client can itself be composed of multiple distributions and only a subset of clients may participate. In this work, our goal is to bring the benefits of generalization to federated learning, while not requiring any prior knowledge of the distributions in the training data.

\begin{algorithm}[!t]
\caption{Federated Averaging \texttt{FedAvg}}
\begin{algorithmic}[1]
{
\Require Devices $k \in K$, local epochs $E$, learning rate $\eta$, global model $w$ randomly initialized
\For{$t = 1 \to T$}
\State Server selects subset $K_t$ of $K$ devices at random
\State Server sends $w^t$ to devices in $K_t$
\For{each client $k \in K_t$ in parallel}
\State Update $w^t$ for $E$ epochs of SGD to get $w_k^{t+1}$
\State Send $w_k^{t+1}$ back to the server
\EndFor
\State Server aggregates $w^{t+1} \leftarrow \sum_{k=1}^K \frac{n_k}{N} w_k^{t+1}$
\EndFor
}
\end{algorithmic}
\label{alg:fedavg}
\end{algorithm}

\section{FedGen: Generalizable Federated Learning}

As described in Section \ref{sec:introduction}, generalizable models need to be able to distinguish between spurious and invariant features, thereby ignoring spurious correlations with the target variable present in training data. Given the set of input features $X$ used by the device models, the set of invariant features $X^I$ across all devices is one where the target prediction probability is consistent across all data distributions, (i.e.) $p(Y | X_i$  $ \in X^I, \mathcal{D})$ is approximately constant. 
Conversely, the set of spurious features $X^S$ consists of features whose prediction probabilities differ across distributions due to the spurious correlations. Hence, it follows that $X^I \cup X^S = X$, and $X^I \cap X^S = \emptyset$, (i.e.) a feature cannot be both invariant and spurious.


Our approach to determining whether the $i^{\text{th}}$ feature $X_i \in X$ is spurious or invariant, is by measuring the stability of its parameter weight $w_i$. Recent work \cite{javed2020learning,venkateswaran2021environment} has shown that if $X_i$ is a causal or invariant feature, $w_i$ converges to a fixed magnitude, (i.e.) $\mathbb{E}[Y|X_i] = c$ for some constant value $c$, across all training epochs. Whereas if $\mathbb{E}[Y|X_i]$ is changing, $w_i$ would keep changing as well, and hence spurious features have parameter weights that exhibit high variance. 
This definition is equivalent to learning features whose correlations with the target variable are stable.

We leverage this intuition for our federated setting and define a set of masks $M_k = \{m_{1k}, ..., m_{jk}\}$ over the input features for each of the $k \in K$ clients, where $m_{ik} \in \mathbb{R}$ and $j$ is the number of input features in $X$. 
We update the masks during training by using the variances in the feature weights to emphasize invariant features and suppress spurious ones. During each training epoch, we update the local masks for each client $k \in K$ as:
\begin{equation}\label{eq:updation_a}
    m_{ik} \gets m_{ik} + \frac{1}{j} \sum_{i=1}^j v(w_{ik}) - \alpha(v(w_{ik})), \; \forall m_{ik} \in M_k
\end{equation} 

where hyper-parameter $\alpha$ serves as a scaling factor, $v(w_{ik})$ is the variance of the weights of feature $X_i$ on client  $k$, and the second term is the average variance of all feature weights on client $k$. From the above intuition, we know that the variance of invariant features is low and that of spurious features would be high. Hence, we see that updating the local masks on each client using Equation \ref{eq:updation_a} results in the masks of invariant features gaining in value on each of the $K$ clients since their variance is lower than the average (i.e.) $ \frac{1}{j} \sum_{i=1}^j v(w_{ik}) - \alpha(v(w_{ik})) > 0$. 
Masks of spurious features on the other hand, become progressively negative, since the variance of their weights coupled with the scaling factor is larger than the average which is brought down by invariant features. 
Since the masks themselves are unbounded, we retain the scale of the original feature values by scaling the masks using the sigmoid function $\sigma$ which is bounded between [$0,1$]. Hence, using the masks and the update function in Equation \ref{eq:updation_a} on the input features for each local client results in:
\begin{equation}\label{eq:representation}
      \sigma(m_{ik}) \odot X_i \rightarrow
  \begin{cases}
    X_i & \text{if $X_i$ is invariant} \\
    0 & \text{if $X_i$ is spurious}
  \end{cases} \; ~~
  ,\forall X_i \in X
\end{equation}

where $\odot$ denotes element-wise multiplication. We then define the local loss on each client in FedGen as:
\begin{gather}\label{eq:fedgen_penalty}
    \underset{w_k}{\mathrm{min}} \; \underbrace{f_k(Z_k, w_k)}_{\text{local loss}} + \underbrace{\lambda \lVert \nabla_{w_k} w_k^\top f_k(Z_k, w_k) \rVert^2_2}_{\text{\texttt{FedGen} penalty}} \; , \; \textrm{where} \\
    f_k(Z_k, w_k) = \mathbb{E}_{(x_k,y_k) \sim D_k}[\ell(Z_k, y_k);w_k] \;, \; \text{s.t.} \; Z_k = \sigma(M_k) \odot x_k
\end{gather}

The penalty term here like IRM serves as a regularizer for the local models that penalizes doing too well on one data distribution (i.e., possibly relying on spurious features) and rewards doing well across distributions (i.e., relying on invariant features).
FedGen then builds a centralized model by communicating and subsequently aggregating both the local model parameters as well as the local masks of the participating devices or clients. We use the same element-wise averaging for aggregation as used by prior work:
\begin{equation}\label{eq:mask_aggregation}
    w^{t+1} = \sum_{k=1}^K \frac{n_k}{N} w_k^{t+1} , \;\; 
    M^{t+1} \leftarrow \sum_{k=1}^K \frac{n_k}{N} M_k^{t+1}
\end{equation}

Aggregating the masks across the $K$ devices or clients allows them to obtain a consensus on invariant and spurious features across the different data distributions that can be present in their local training data. Masks that are positive (i.e. reflecting invariant features) across most clients will continue to be emphasized after aggregation, and likewise for the negative spurious feature masks. The mask aggregation also allows client models relying on the wrong features to obtain the global consensus and ignore spurious features. We define the FedGen central objective function as:
\begin{equation}\label{eq:fedgen_central_objective}
    \min_w F(Z ,w) \;, \; \text{where} \; F(Z, w) = \sum_{k=1}^K \frac{n_k}{N} f_k(Z_k, w_k)
\end{equation}

We summarize our FedGen approach in Algorithm \ref{alg:fedgen}.
We next formally prove that FedGen results in an aggregated model that is generalizable in Theorem \ref{thm:mask} by formulating a minimax problem and showing that the resulting model minimizes the loss using invariant features, even under the most adverse unseen test distributions. 

\begin{algorithm}[!t]
\caption{Federated Generalization \texttt{FedGen}}\label{alg:fedgen}
\begin{algorithmic}[1]
\Require Devices $k \in K$, local epochs $E$, learning rate $\eta$, global model $w$ randomly initialized
\Require Local masks $M_k \in M$, scaling factor $\alpha$, warm up epochs $e^{init}$
  \For{$t = 1 \to T$}
    \State Server selects subset $K_t$ of $K$ devices at random
    \State Server sends $w^t$ to devices in $K_t$
  \EndFor
  \State \textbf{ClientUpdate:}
  \State For each client $k \in K$,
  \State Initialize masks $m_{ik} = 1,  \; \forall m_{ik} \in M_k$ 
  \State Initialize mean $\textbf{u}_{ik} = 0$, variance $\textbf{v}_{ik} = 0$
\For{$e = 1 \to E$}
\State Sample batch $x_k, y_k$
\State $ \ell_{loc}=[\ell(\sigma(M_k) \odot x_k,y_k);w_k]$ \algorithmiccomment{prediction error}
\State $\ell_1 = \lVert w_k \rVert_1$ \algorithmiccomment{$L1$ regularization}
\State $\ell_{pen} = \lambda \lVert \nabla_{w_k} w_k^\top f_k(Z_k, w_k) \rVert^2_2$ \algorithmiccomment{FedGen penalty}
\State $\mathcal{L} = \ell_{loc} + \ell_{1} + \ell_{pen}$ \algorithmiccomment{Total loss}
\State $w_k = w_k - \eta\nabla w_k \mathcal{L}$
\State $\textbf{u}_{old} = \textbf{u}_{k}$
\State $\textbf{u}_k = \beta w_k + (1-\beta)\textbf{u}_{old}$ \algorithmiccomment{Mean estimate}
\State $\textbf{v}_k = \delta \textbf{v}_k + (1-\delta)(w_k - \textbf{u}_{old})^2$ \algorithmiccomment{Variance Estimate}
\If{$e > e^{init}$} \algorithmiccomment{Update masks}
\State $m_{ik} \mathrel{+}= \frac{1}{j} \sum_{i=1}^j \textbf{v}_k(w_{ik}) - \alpha(\textbf{v}_k(w_{ik})), \forall m_{ik} \in M_k$ 
\EndIf
\EndFor
\State Server aggregates $w^{t+1} \leftarrow \sum_{k=1}^K \frac{n_k}{N} w_k^{t+1}$
\State Server aggregates $M^{t+1} \leftarrow \sum_{k=1}^K \frac{n_k}{N} M_k^{t+1}$
  \end{algorithmic}
\end{algorithm}

\begin{theorem}\label{thm:mask}
Given training distributions $D^{tr}$ and a test distribution $D^{te}$, the set of invariant features $X^I$ is the saddle point of the following minimax problem:
\begin{equation*}
    X^I = \underset{w}{\textup{min}} \; \underset{X^I, X^S}{\textup{max}} F(Z, w ; D^{te}) \;\;, \; \textup{where} \;\; Z = \sigma(M) \odot X \;,
\end{equation*}
where $F(Z, w)$ is the FedGen central objective function, $M$ the final aggregated masks, and $X^I, X^S$ denote the set of invariant and spurious features respectively such that $X^I \cup X^S = X$, and $X^I \cap X^S = \emptyset$ (i.e.) they are disjoint.
\end{theorem}

\begin{proof}

For every $Z$, we can partition it into invariant variables $Z^I$ and non-invariant variables $Z^S$ as:
\begin{equation} \label{eq:zexpressedinsigmaxi}
    Z^I = \sigma(M) \odot X^I\ , \;\;\;\; Z^S = \sigma(M) \odot X^S.
\end{equation}

Consider a test distribution $D^{*te}$, where the set of spurious features $X^{S*}$ are not predictive of the output $Y$, and only the invariant features $X^I$ are predictive of $Y$, (i.e.)
\begin{equation}\label{eq:invariant_predictive}
    p(Y | Z, D^{*te}) = p(Y | Z^I, D^{*te}) , \;\; p(Y | Z, D^{tr}) = p(Y | Z^I, D^{tr})
\end{equation}

\noindent Therefore, $F(Z, w\; ; D^{*te})$
  \begin{align}\label{eq:proof_reduction}
  \begin{split}
    &= H(p(Y | Z, D^{*te}) ; p(Y | Z, D^{tr})) \\
    &\overset{(i)}{=} H(p(Y | Z^I , D^{*te}) ; p(Y | Z^I, D^{tr})) \\
    &\overset{(ii)}{=} H(p(Y | \sigma(M)\odot X^I , D^{*te}) ; p(Y | \sigma(M)\odot X^I, D^{tr})) \\
    &\overset{(iii)}{=} H(p(Y | X^I , D^{*te}) ; p(Y | X^I , D^{tr})) \\
    &= F(X^I, w \; ; D^{*te})
    \end{split}
  \end{align}
  
\noindent where $H(\cdot)$ is the cross-entropy loss function. Step $(i)$ is obtained from applying equation (\ref{eq:invariant_predictive}). Step $(ii)$ is obtained by applying equation (\ref{eq:zexpressedinsigmaxi}), and step $(iii)$ is due to the property of the masks from equation (\ref{eq:representation}).

Recall that $X^{S*}$ was assumed to be non-predictive of $Y$. However, in most cases, the spurious feature $X^S$ would have some predictive power over $Y$ in the training environment.
Hence, from the definition of spurious features $X^S$, their biased influence on the model performance during training will lead to an increased loss in the worst case test environment:
\begin{equation}\label{eq:optimize_spurious_part1}
    \underset{X^S}{\textup{max}} \; F(Z,w ; D^{te}) \geq F(Z,w ; D^{*te}) 
\end{equation}

Recall $X^I$ denotes the set of invariant features, thus  $p(Y | X^I, D^{te})$ does not depend on $X^S$. Therefore, 
\begin{equation}\label{eq:optimize_spurious_part2}
    \underset{X^S}{\textup{max}} \; F(X^I, w ; D^{te}) = F(X^I, w ; D^{*te}) 
\end{equation}

By combining equations (\ref{eq:proof_reduction}), (\ref{eq:optimize_spurious_part1}), and (\ref{eq:optimize_spurious_part2}), we have:
\begin{equation}\label{eq:max_invariant}
    \underset{X^S}{\textup{max}} \; F(Z, w ; D^{te}) \geq \underset{X^S}{\textup{max}} \; F(X^I, w ; D^{te})
\end{equation}

The above formulation holds for all $X^I$. Hence, taking the maximum over $X^I$ in equation (\ref{eq:max_invariant}) preserves the inequality, 
\begin{equation*}
    \underset{X^I, X^S}{\textup{max}} F(Z, w ; D^{te}) \geq \underset{X^I, X^S}{\textup{max}} F(X^I, w ; D^{te})
\end{equation*}
which in turn implies,
\begin{equation*}
    X^I = \underset{w}{\textup{min}} \underset{X^I, X^S}{\textup{max}} F(Z,w ; D^{te})
\end{equation*}
\end{proof}

\section{Convergence Analysis of FedGen}

In this section, we analyze the convergence of FedGen by first introducing some definitions and assumptions for our convergence analysis. We subsequently show the expected convergence of FedGen under both convex and non-convex conditions.

\begin{definition}
(Smoothness)
The function f is L-smooth with constant L $>$ 0 if $\forall x_1, x_2$
\begin{equation}
    f(x_1) - f(x_2) \leq \langle \nabla f(x_2), x_1 - x_2 \rangle + \frac{L}{2} || x_1 - x_2 ||^2
\end{equation}
\end{definition}

\begin{definition}
(Convexity)
The function f is $\mu$-strongly convex with $\mu > 0$ if $\forall x_1, x_2$
\begin{equation}
    f(x_1) - f(x_2) \geq \langle \nabla f(x_2), x_1 - x_2 \rangle + \frac{\mu}{2} || x_1 - x_2 ||^2
\end{equation}
\end{definition}

In order to analyze convergence in heterogeneous non-iid settings with constant step-size (as is usually implemented in practice), we also introduce a notion of dissimilarity between the local device functions in the federated network similar to assumptions made by prior work \cite{li2018federated,vaswani2019fast,yin2018gradient}. 

\begin{definition}
(B-local dissimilarity)
The local functions $f_k$ are B-locally dissimilar at $Z, w$ if $\mathbb{E}||(\nabla f_k(Z, w)||^2 \leq ||\nabla F(Z, w)||^2 B^2$. 
\end{definition}
We further define $B(w) = \sqrt{\frac{\mathbb{E}||\nabla f_k(Z,w)||^2}{||\nabla F(Z, w)||^2}}$ when $|| \nabla F(Z, w) ||^2 \neq 0$. As explained in \cite{li2018federated}, when all the local functions are the same (i.e.) distributions of data sampled from the devices are i.i.d, we have $B(w) \rightarrow 1$ for all $w$. However, in the presence of spurious correlations, the data distributions are heterogeneous and therefore $B > 1$, and larger the value of $B(w)$, the more the heterogeneity of the local data. We next make the following assumption on the objective functions.

\begin{assumption}\label{ass:fedgen}
Assume that:
\begin{enumerate}
    \item The FedGen central objective function F(Z, w) is bounded from below, (i.e.) $\min_w F(Z, w) > -\infty$
    \item There exists $\epsilon > 0$ such that $\mathbb{E}[\nabla f_k(Z, w)] \leq || \nabla F(Z, w) ||$, and $\nabla F(Z, w)^\top \mathbb{E}[\nabla f_k(Z, w)] \geq \epsilon || \nabla F(Z, w) ||^2$ holds for all w.
\end{enumerate}
\end{assumption}

We note that if $\epsilon = 1$, then $\nabla f_k(Z, w)$ is an unbiased estimator of $\nabla F(Z, w)$


\begin{theorem}\label{thm:fedgen_convex_convergence}
Let Assumption \ref{ass:fedgen} hold. Assume F(Z, w) is L-smooth and $\mu$-strongly convex and functions $f_k$ are B-locally dissimilar with learning rate $\eta$. Then we can show that for $\rho = -\eta(\epsilon - \frac{L\eta B^2}{2})$, at iteration $t$ of FedGen, we have the following expected decrease in the global objective:
\begin{multline*}
    \mathbb{E}[F(Z^{t+1}, w^{t+1})] - F(Z_*, w_*) \leq \\ (1 - 2 \mu \rho) (F(Z^t, w^t) - F(Z_*, w_*))
\end{multline*}
\end{theorem}

\begin{proof}
Since $F(Z, w)$ is L-smooth, we have
\begin{align}
    \begin{split}
    &F(Z^{t+1}, w^{t+1}) - F(Z^t, w^t) \\
    &\leq \langle \nabla F(Z^t, w^t), w^{t+1} - w^t \rangle + \frac{L}{2} || w^{t+1} - w^t ||^2 \\
    &\leq -\nabla F(Z^t, w^t)^\top \eta \nabla f_k(Z^t, w^t) + \frac{L}{2} || \eta \nabla f_k(Z^t, w^t) ||^2
    \end{split}
\end{align}
With Assumption 1 and B-local dissimilarity, we get
\begin{align}\label{eq:convex_convergence}
    \begin{split}
        &\mathbb{E}[F(Z^{t+1}, w^{t+1})] - F(Z^t, w^t) \\ 
        &\leq -\eta \nabla F(Z^t, w^t)^\top \mathbb{E}[\nabla f_k(Z^t, w^t)] + \frac{L\eta^2}{2} \mathbb{E}[ || \nabla f_k(Z^t, w^t) ||^2] \\
        &\leq -\eta \epsilon || \nabla F(Z^t, w^t) ||^2 + \frac{L \eta^2 B^2}{2} || \nabla F(Z^t, w^t) ||^2 \\
        &\leq -\eta(\epsilon - \frac{L \eta B^2}{2}) || \nabla F(Z^t, w^t) ||^2
    \end{split}
\end{align}

Setting $\rho = -\eta(\epsilon - \frac{L\eta B^2}{2})$ we get, 
\begin{equation}\label{eq:L_smooth_complete}
    \mathbb{E}[F(Z^{t+1}, w^{t+1})] - F(Z^t, w^t) \leq  - \rho ||\nabla F(Z^t, w^t)||^2
\end{equation}

Since $F(Z, w)$ is $\mu$-strongly convex, for some $Z'$ and $w'$, we get:
\begin{equation}\label{eq:mu_convex_first}
    F(Z', w') - F(Z^t, w^t) \geq \nabla F(Z^t, w^t)^\top (w' - w^t) + \frac{\mu}{2} || w' - w^t ||^2
\end{equation}

Since $F(Z', w')$ is quadratic in $w'$, its minimal value is obtained when $\nabla F(Z', w') = \nabla F(Z^t, w^t) + \mu (w' - w^t) = 0$. 
This occurs when $w' = w^t - \frac{\nabla F(Z^t, w^t)}{\mu}$. 
Substituting for $w'$ in Equation \ref{eq:mu_convex_first} we get,
\begin{equation}
    F(Z_*, w_*) \geq F(Z', w') \geq F(Z^t, w^t) -  \frac{|| \nabla F(Z^t, w^t) ||^2}{2\mu}
\end{equation}

\begin{equation}
    2 \mu (F(Z^t, w^t) - F(Z_*, w_*)) \leq || \nabla F(Z^t, w^t) ||^2
\end{equation}

Substituting into Equation \ref{eq:L_smooth_complete} we get, 
\begin{equation}
    \mathbb{E}[F(Z^{t+1}, w^{t+1})] - F(Z^t, w^t) \leq  - 2 \mu \rho (F(Z^t, w^t) - F(Z_*, w_*))
\end{equation}
Re-arranging and subtracting $F(Z_*, w_*)$ from both sides, 
\begin{multline}
    \mathbb{E}[F(Z^{t+1}, w^{t+1})] - F(Z_*, w_*) \leq \\ (1 - 2 \mu \rho) (F(Z^t, w^t) - F(Z_*, w_*))
\end{multline}

\end{proof}


\begin{theorem}\label{thm:fedgen_nonconvex_convergence}
Let Assumption \ref{ass:fedgen} hold. Assume F(Z, w) is L-smooth and non-convex and functions $f_k$ are B-locally dissimilar with learning rate $\eta$. Also, let $F(Z^0, w^0) - F(Z^*, w^*) = \Updelta$, where $w^* = \min_w F(Z, w)$. Then after $T$ iterations of FedGen, we have
\begin{equation}
    \frac{1}{T}\sum_{t=0}^{T-1} \rho \mathbb{E}[|| \nabla F(Z^t, w^t) ||^2] \leq \Updelta
\end{equation}
\end{theorem}

\begin{proof}
By taking the full expectation of Equation \ref{eq:convex_convergence}, we get
\begin{equation}\label{eq:non-convex_1}
    \mathbb{E}[F(Z^{t+1}, w^{t+1})] \leq \mathbb{E}[F(Z^t, w^t)] - \rho \mathbb{E}[|| \nabla F(Z^t, w^t) ||^2]
\end{equation}
Then summing up Equation \ref{eq:non-convex_1} over global iteration T, we obtain
\begin{equation}
    \mathbb{E}[F(Z^{t+1}, w^{t+1})] \leq F(Z^0, w^0) - \sum_{t=0}^{T-1} \rho \mathbb{E}[|| \nabla F(Z^t, w^t) ||^2]
\end{equation}
From Assumption 1, we see that $F(Z^*, w^*) \leq \mathbb{E}[F(Z^{t+1}, w^{t+1})]$, using which we can write
\begin{equation}
    F(Z^*, w^*) \leq F(Z^0, w^0) - \frac{1}{T}\sum_{t=0}^{T-1} \eta \rho \mathbb{E}[|| \nabla F(Z^t, w^t) ||^2]
\end{equation}
Re-arranging,
\begin{equation}
    \frac{1}{T}\sum_{t=0}^{T-1} \rho \mathbb{E}[|| \nabla F(Z^t, w^t) ||^2] \leq F(Z^0, w^0) - F(Z^*, w^*)
\end{equation}
\end{proof}





\section{Experiments}\label{sec:experiments}

In this section, we evaluate FedGen on diverse models, clients, and datasets spanning multiple application domains: Natural Language Processing (NLP), Internet-of-Things (IoT), and business process mining (Table \ref{tab:dataset}). Similar to prior work \cite{arjovsky2019invariant, ahuja2020invariant, choe2020empirical, krueger2021out}, we augment these benchmark datasets with spurious features $X^S$ such that it has a strong spurious correlation with the target labels $Y$ in the training data distributions (i.e.) $p(Y | X^S, D^{tr}) \geq \alpha = 0.8$, and this spurious relationship does not hold in the test distribution (i.e.) $p(Y | X^S, D^{te}) \leq \alpha = 0.1$.

We assess the quality of generalization of the aggregated model as the classification accuracy obtained on the test environment, disjoint from the set of training environments. Hence, an approach that cannot generalize will do well during training but not on the test set, while a generalizable model will continue to perform well on the test data. 
We also present results from an ablation study of the masking function as well as additional experiments that highlight the robustness of our approach. 

\begin{table}[!t]
\centering
\resizebox{0.48\textwidth}{!} {%
    \begin{tabular}{lccccc} \toprule
    Dataset & Devices/Clients & Train Seq. & Test Seq. & \#Spurious
    \\ 
     \midrule
     SST-2 & 100 & 67,350 & 873 & 1 \\
     AG News & 100 & 120,000 & 7600 & 1 \\
     HAR & 30 & 7352 & 2947  & 2\\
     Water Quality & 120 & 132,227 & 14,237 & 1 \\
     BPIC 2018 & 100 & 306,615 & 63,692 & 1  \\
     \bottomrule
    \end{tabular}%
    }
    \caption{Summary of Federated Datasets}
    \label{tab:dataset}
\end{table}

\begin{table*}[!t]
\centering
\resizebox{\textwidth}{!} {%
    \begin{tabular}{cccccccccccc} \toprule
    \textbf{Algorithm} &
    \multicolumn{2}{c}{SST-2} &
    \multicolumn{2}{c}{AG News} &
    \multicolumn{2}{c}{HAR} &
    \multicolumn{2}{c}{Water Quality} &
    \multicolumn{2}{c}{BPIC 2018}
    \\ 
     & Train & Test & Train & Test & Train & Test \\
     \midrule
     ERM & 85.6 $\pm$ 0.1 & 9.6 $\pm$ 0.3 & 92.3 $\pm$ 0.1 & 63.4 $\pm$ 0.5 & 98.1 $\pm$ 0.3 & 73.2 $\pm$ 0.5 & 88.7 $\pm$ 0.3 & 65.3 $\pm$ 0.3 & 77.1 $\pm$ 0.7 & 73.7 $\pm$ 0.4 \\
     
     FedAvg & 83.3 $\pm$ 0.1 & 9.6 $\pm$ 0.1 & 91.5 $\pm$ 0.6 & 61.3 $\pm$ 0.2 & 98.0 $\pm$ 0.5 & 70.1 $\pm$ 1.6 & 89.3 $\pm$ 0.7 & 63.8 $\pm$ 1.6 & 77.1 $\pm$ 2.1 & 71.7 $\pm$ 1.3 \\
     
     FedProx & 75.0 $\pm$ 0.1 & 9.6 $\pm$ 0.1 & 91.3 $\pm$ 0.8 & 61.1 $\pm$ 0.3 & 96.7 $\pm$ 0.5 & 71.5 $\pm$ 0.6 & 87.7 $\pm$ 1.1 & 66.1 $\pm$ 0.7 & 73.9 $\pm$ 1.1 & 71.2 $\pm$ 0.5 \\
     
     \bf{FedGen (ours)} & 58.3 $\pm$ 0.2 & \textbf{60.9 $\pm$ 0.3} & 81.8 $\pm$ 0.3 & \bf{82.3 $\pm$ 0.1} & 94.2 $\pm$ 1.2 & \bf{87.5 $\pm$ 0.3} & 64.3 $\pm$ 0.2 & \bf{85.6 $\pm$ 1.1} & 74.8 $\pm$ 0.8 & \bf{86.1 $\pm$ 0.4}\\ 
     
     \midrule
     
     Inv-FedAvg & 66.6 $\pm$ 0.2 & 62.5 $\pm$ 0.2 & 82.1 $\pm$ 0.4 & 82.8 $\pm$ 0.1 & 93.0 $\pm$ 0.4 & 89.9 $\pm$ 0.5 & 74.9 $\pm$ 1.1 & 88.1 $\pm$ 0.6 & 78.0 $\pm$ 0.3 & 88.7 $\pm$ 0.3 \\
     \bottomrule
    \end{tabular}
    }
    \caption{Accuracy achieved by comparison approaches for all datasets}
    \label{tab:accuracy_results}
\end{table*}

\subsection{Datasets}

\subsubsection{SST-2:}  We modify the Stanford Sentiment Treebank (SST-2) \cite{socher2013recursive}, a benchmark dataset for binary sentiment analysis with spurious correlations in a similar manner as \cite{choe2020empirical}. 
We add a spurious feature $X^S$ by pairing each sentiment with a final punctuation mark (. or !) in the corresponding text as follows:
\begin{align*}
p(Y = 0 \;|\; X^S = . \;, D) = p( Y = 1 \;|\; X^S = \;! \;, D) &= \alpha \\
p(Y = 0 \;|\; X^S = \;! \;, D) = p(Y = 1 \;|\; X^S = . \;, D) &= 1 - \alpha
\end{align*}
where we set $\alpha$ as described above. 
The addition of the punctuation mark as a separate feature, and not as part of the sentence, allows us to isolate the sentence sequence embeddings to mask, while still preserving the structure and meaning of the sentence \cite{choe2020empirical}.
We create two different training distributions with differing strengths of spurious correlations and treat the test data as a third distribution with weak spurious correlation as described above. We distribute the training data among 100 devices or clients. 

\subsubsection{AG News:} To evaluate FedGen for multi-class NLP predictions, we use the AG News dataset \cite{zhang2015character}, which consists of a corpus of news articles and their titles that have been classified into four categories - (1) World, (2) Sports, (3) Business, and (4) Sci/Tech. 
There are 30,000 training and 1,900 test examples for each class. 
We again add a spurious feature that pairs each of the four news categories with a corresponding final punctuation mark in the article from the set $\mathcal{P} = \{. \; ! \; $@$ \; \wedge\}$ as follows:

\begin{align*}
p(Y = 0 \;|\; X^S = . \;, D) = p( Y = 1 \;|\; X^S = ! \;, D) &= \alpha \\
p(Y = 2 \;|\; X^S = @ \;, D) = p( Y = 3 \;|\; X^S = \wedge \;, D) &= \alpha \\
p(Y = y_j \;|\; X^S = \text{rand}(\mathcal{P} \setminus \mathcal{P}_j) \;, D) &= 1 - \alpha
\end{align*}
and set $\alpha$ as before. We divide the training data into two distributions with strong spurious correlations, and a separate test distribution with a weak correlation. We distribute the training data among 100 devices.

\subsubsection{Human Activity Recognition (HAR)} 
The HAR dataset \cite{anguita2013public} consists of smartphone accelerometer and gyroscope readings recorded from 30 individuals performing six activities (walking, standing, sitting, etc.). The model objective is to predict the activity being performed given the sensor readings. 
The data consists of sequences of 128 time-steps of sensor readings which correspond to any particular activity. 
We consider the data from the smartphone of each of the individuals to be a client, and divide them into two training and one test distribution such that there is a strong spurious correlation between the specific activity performed and the smartphone model used by individuals in the training test, which does not hold for individuals in the test set as follows:
\begin{align*}
p(Y = y_j \;|\; X^S = S_j \;, D) &= \alpha \\
p(Y = y_j \;|\; X^S = \text{rand}(S \setminus S_j) \;, D) &= 1 - \alpha
\end{align*}
where $y_j$ and $S_j$ refer to the $j$th activity and smartphone model respectively.

\subsubsection{Water Quality Detection}
The dataset \cite{epagov} comprises of water quality data collected from different cities across the United States. It includes sensor readings such as turbidity, conductivity, pH, temperature, DO$_2$, in addition to environmental attributes like humidity and precipitation, which correspond to binary labels depicting the occurrence of a contamination event. We divide the training data into two distributions, clustered the humidity into low ($L$) and high ($H$) values and spurious correlate them to the contamination labels ($\mathcal{C}, \emptyset$) as follows:
\begin{align*}
p(Y = \mathcal{C} \;|\; X^S \in H \;, D) = p( Y = \emptyset \;|\; X^S \in L \;, D) &= \alpha \\
p(Y = \mathcal{C} \;|\; X^S \in L \;, D) = p(Y = \emptyset \;|\; X^S \in H \;, D) &= 1 - \alpha
\end{align*}
We distribute the training data among 120 clients corresponding to the different cities.

\subsubsection{Business Process Mining:}
We use a benchmark process mining dataset (BPIC 2018)\footnote{https://data.4tu.nl/articles/BPI\_Challenge\_2018/12688355} which consists of payment applications from German farmers to the European Agricultural Guarantee Fund collected over a period of three years. 
Each application is processed by one of four departments, and to evaluate our approach in a multi-distribution setting, we consider each department as a distribution and use three of them for training and the fourth as our out-of-distribution (OOD) test data. We consider two business process variants and augment the process trace sequences with a spurious feature denoting the area of the farm. We cluster the area into low ($L$) and high ($H$) values and spuriously correlate them to the two business process variants ($\mathcal{A}, \mathcal{B}$) as follows:
\begin{align*}
p(Y = \mathcal{A} \;|\; X^S \in H \;, D) = p( Y = \mathcal{B} \;|\; X^S \in L \;, D) &= \alpha \\
p(Y = \mathcal{A} \;|\; X^S \in L \;, D) = p(Y = \mathcal{B} \;|\; X^S \in H \;, D) &= 1 - \alpha
\end{align*}
and define $\alpha$ as before. We distribute the training data among 100 clients.

\subsection{Benchmarks}
We compare FedGen to four benchmark approaches:
\begin{itemize}
    \item \textbf{ERM:} Standard empirical risk minimization approach that minimizes the average loss over the entire training data. We use this as a benchmark comparison against standard centralized training. 
    
    \item \textbf{FedAvg:} Federated averaging approach by \cite{mcmahan2017communication} where model parameters are averaged element-wise during aggregation.
    
    \item \textbf{FedProx:} Approach by \cite{li2018federated} to handle heterogeneous non-iid data in federated settings using a proximal term to keep local models close to the global model. 
    
    \item \textbf{Inv-FedAvg:} FedAvg trained on data without any spurious features. This approach reflects the setting where spurious correlations are not present and hence reflects the best performance.
\end{itemize}

\subsection{Implementation Details}
We implement FedGen, ERM, FedAvg, and Inv-FedAvg using the PyTorch library \cite{paszke2019pytorch} and implement FedProx based on their publicly available code. 
We simulate the federated learning setup (1 server and $N$ devices) on a commodity machine with 7 Intel\textsuperscript{\textregistered} Xeon\textsuperscript{\textregistered} E5-2690 v4 CPUs and an NVIDIA\textsuperscript{\textregistered} Tesla P100 GPU.  
We used a 3-layer MLP (50-50-2) with ReLU activations and the cross-entropy loss for classification during training with the Adam optimizer \cite{kingma2014adam}. 
We identify the final models when they have either converged, started to diverge, or run a sufficient number of rounds, whichever comes earlier, similar to definitions by prior work \cite{li2018federated, sahu2018convergence}. 
We also assume the central server samples all the local clients to join the training process in every communication round for simplicity. 

\subsection{Results}

\begin{figure*}[!t]
\centering
\begin{minipage}[b]{.195\linewidth}
\centering
\includegraphics[width=\textwidth]{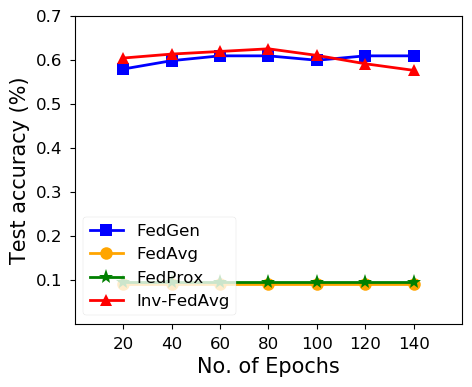}
\subcaption{SST-2}
\label{fig:sst_epochs}
\end{minipage}%
\begin{minipage}[b]{.195\linewidth}
\centering
\includegraphics[width=\textwidth]{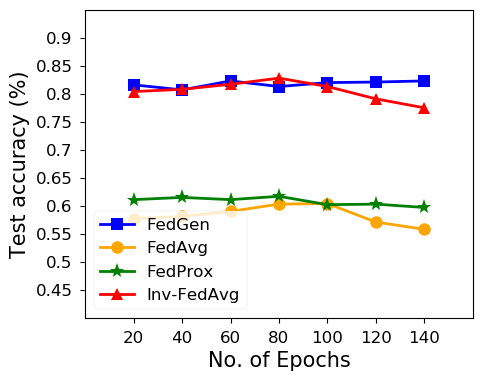}
\subcaption{AG News}
\label{fig:ag_epochs}
\end{minipage}%
\begin{minipage}[b]{.195\linewidth}
\centering
\includegraphics[width=\textwidth]{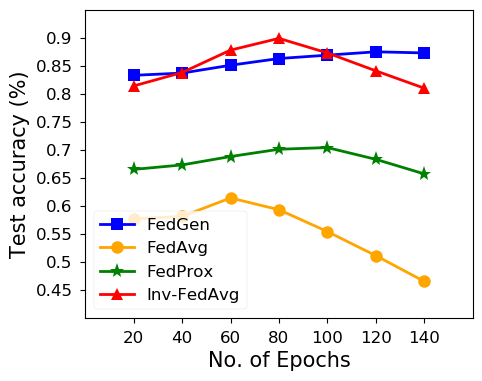}
\subcaption{HAR}
\label{fig:har_epochs}
\end{minipage}%
\begin{minipage}[b]{.195\linewidth}
\centering
\includegraphics[width=\textwidth]{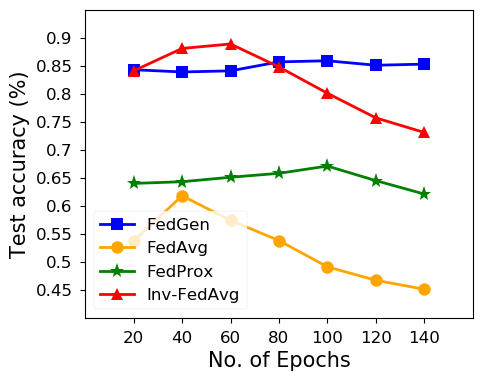}
\subcaption{Water Quality}
\label{fig:stormwater_epochs}
\end{minipage}
\begin{minipage}[b]{.195\linewidth}
\centering
\includegraphics[width=\textwidth]{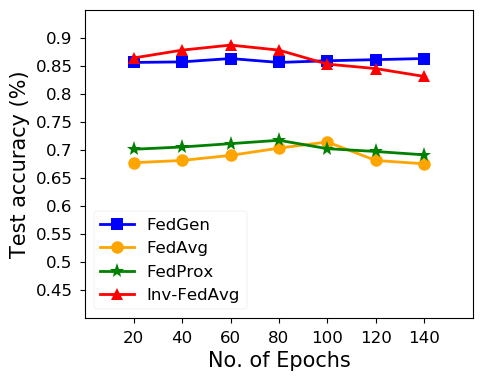}
\subcaption{BPIC 2018}
\label{fig:airq_epochs}
\end{minipage}
\caption{Effect of number of local training epochs on the different approaches across all datasets}
\label{fig:effect_of_epochs}
\end{figure*}

\begin{figure*}[!t]
\centering
\begin{minipage}[b]{.195\linewidth}
\centering
\includegraphics[width=\textwidth]{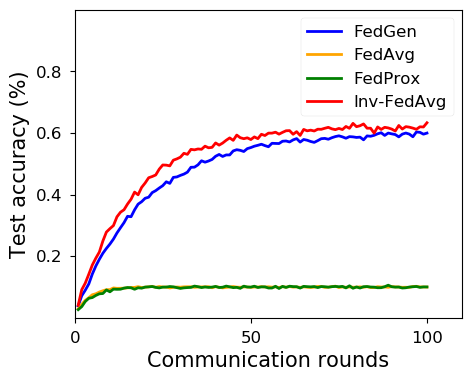}
\subcaption{SST-2}
\label{fig:sst2_convergence}
\end{minipage}%
\begin{minipage}[b]{.195\linewidth}
\centering
\includegraphics[width=\textwidth]{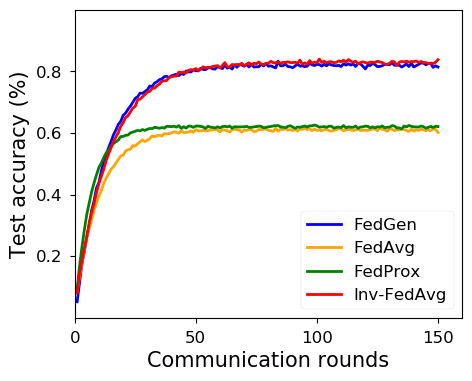}
\subcaption{AG News}
\label{fig:agnews_convergence}
\end{minipage}%
\begin{minipage}[b]{.195\linewidth}
\centering
\includegraphics[width=\textwidth]{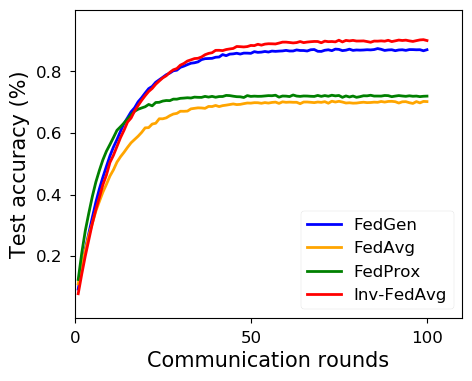}
\subcaption{HAR}
\label{fig:har_convergence}
\end{minipage}%
\begin{minipage}[b]{.195\linewidth}
\centering
\includegraphics[width=\textwidth]{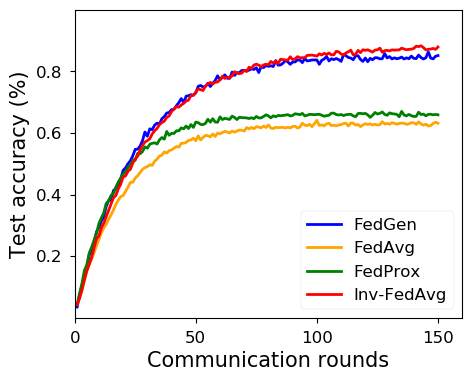}
\subcaption{Water Quality}
\label{fig:stormwater_convergence}
\end{minipage}
\begin{minipage}[b]{.195\linewidth}
\centering
\includegraphics[width=\textwidth]{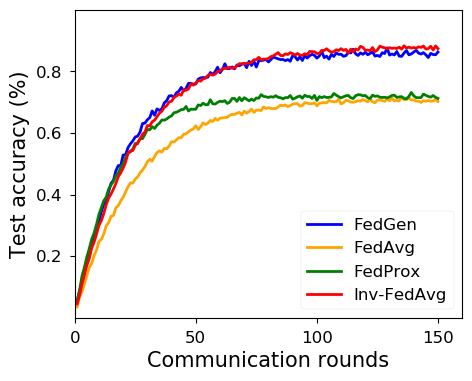}
\subcaption{BPIC 2018}
\label{fig:bpm_convergence}
\end{minipage}
\caption{Convergence rates of the different approaches across all datasets}
\label{fig:convergence_rate}
\end{figure*}

\subsubsection{Test Accuracy}

We first measure the accuracy of the final aggregated global model resulting from all the comparison approaches on the out-of-distribution (OOD) test datasets to compare their generalizability. 
Table \ref{tab:accuracy_results} summarizes the results, where we see that ERM, FedAvg, and FedProx, are all heavily influenced by the spurious correlations since their training accuracy across the datasets is high. However, this results in low generalization to the test set where the spurious correlations no longer hold. 

We observe that FedGen on the other hand, ignores the spurious correlations (lower training accuracy), and instead identifies and relies on the invariant features thereby achieving good generalization and high test accuracy across all the datasets. It outperforms ERM, FedAvg, and FedProx, and achieves $24\%$ more accuracy on average across all five datasets. Also importantly, FedGen, even in the presence of spurious correlations, performs nearly as well as Inv-FedAvg (within $2\%$ accuracy on average) which reflects the accuracy when trained on data without any spurious correlations. 

In addition, 
even though FedAvg and Inv-FedAvg rely on the same model and algorithm, the additional spurious features results in a significant increase in FedAvg's training accuracy, which does not translate to generalizability on unseen test distributions, highlighting the importance of building generalizable model training approaches.

\subsubsection{Effect of choice of local training epochs}

While a large number of local training epochs $E$ can help in reducing the communication costs, previous work has also shown that the value of $E$ can impact the performance of FedAvg and can sometimes lead to divergence \cite{mcmahan2017communication, caldas2018leaf, sahu2018convergence}.
We conduct an experimental study on the effect of the choice of $E$ over the different comparison approaches across all five datasets. 
The candidate local epochs we consider are $E \in \{20,40,60,80,100,120,140\}$.
We run each approach till it achieves convergence and report the test accuracy achieved in Figure \ref{fig:effect_of_epochs}. 
We observe that the performance of FedAvg and Inv-FedAvg deteriorates with longer local training, thus showcasing its sensitivity to hyperparameter settings, which matches prior observations made in literature. 
FedProx only partially alleviates this problem and also demonstrates a drop in accuracy for larger number of epochs. 
On the other hand, we observe that FedGen often benefits from longer training, since it gets to further emphasize invariant features. 


\subsubsection{Convergence Rate}
In this experiment we compare the global model convergence achieved by all the comparison approaches. 
We report the convergence rate for the local training epochs $E$ that yields the best final aggregated model accuracy over the test set for each of the approaches. 
Figure \ref{fig:convergence_rate} shows the convergence rates of all the approaches across the five datasets. We observe that the convergence rate achieved by the approaches is quite similar, and that FedProx converges the quickest, followed by FedGen and FedAvg. This demonstrates that FedGen can result in significantly more generalizable models without a loss in runtime performance compared to existing approaches.

\subsubsection{Ablation Study}

We perform an ablation experiment to measure the contribution of the components of our masking function and the FedGen penalty term as shown in Table \ref{tab:fedgen_ablation}. 
We first remove the hyperparameter $\alpha$ used to scale the masks, and observe that there is a drop in accuracy across all three datasets since the degree of spuriousness may not be fully captured, and hence the local models can get influenced by spurious features. 
We next remove the masking function $\sigma(\textbf{M})$, essentially reducing the loss function to that used by FedAvg. 
This is reflected by the accuracy values, which are similar to those achieved by FedAvg, thereby demonstrating that our proposed masking function is the primary reason for model generalization. 
We finally remove the penalty term while optimizing model weights (Equation \ref{eq:fedgen_penalty}) and again observe a drop in accuracy across all datasets since the regularization is important to ensure that the model does not get influenced by spurious correlations in any one distribution.

\begin{table}[!t]
\centering
\resizebox{0.49\textwidth}{!} {%
    \begin{tabular}{lccccc} \toprule
    \textbf{Algorithm} &
    \multicolumn{5}{c}{Test Accuracy (\%)}
    \\ 
     & SST-2 & AG News & HAR & Water Quality & BPIC 2018 \\
     \midrule
     FedGen & 60.9 & 82.3 & 87.5 & 85.6 & 86.1 \\
     $-$ scaling ($\alpha$) & 9.6 & 55.6 & 74.9 & 65.5 & 71.2 \\
     $-$ $\sigma(\textbf{M})$ & 9.6 & 61.1 & 68.1 &  63.4 & 70.0 \\
     $-$ penalty & 9.6 & 59.7 & 70.4 & 61.3 & 69.3 \\

     \bottomrule
    \end{tabular}
    }
    \caption{Ablation Study}
    \label{tab:fedgen_ablation}
\end{table}

\begin{table}[!t]
\centering
\resizebox{0.31\textwidth}{!} {%
    \begin{tabular}{lcccc} \toprule
    \textbf{Algorithm} &
    \multicolumn{4}{c}{Number of Distributions}
    \\ 
     & 2 & 4 & 6 & 8 \\
     \midrule
     ERM & 63.4 & 63.0 & 62.6 & 62.0 \\
     FedAvg & 61.3 &  61.0 & 60.3 & 60.0 \\
     FedProx & 61.1 & 60.7 & 60.6 & 60.3 \\
     \bf{FedGen (ours)} & \bf{82.3} & \bf{82.1} & \bf{82.1} & \bf{82.2} \\
     \midrule
     Inv-FedAvg &  82.8 & 82.8 & 82.6 & 82.8 \\
     
     \bottomrule
    \end{tabular}
    }
    \caption{Varying Number of Training Distributions}
    \label{tab:environments}
\end{table}

\subsubsection{Varying Number of Training Distributions}

We measure the robustness of FedGen to varying numbers of distributions in the training data using the AG News dataset. We augment the dataset with more distributions with differing levels of spurious correlations with $\alpha > 0.8$. Table \ref{tab:environments} shows that the performance of FedGen and Inv-FedAvg remains consistent even with an increased number of distributions, demonstrating FedGen's robustness, while the other approaches are impacted to a small extent.  
\section{Related Work}\label{sec:relatedwork}

As devices like phones, IoT sensors, and wearables grow in capabilities and popularity, the use of federated learning has become increasingly attractive to train models \cite{li2020federated}. To improve the performance of the popular FedAvg approach \cite{mcmahan2017communication} in statistically heterogeneous settings, FedProx \cite{li2018federated} and Agnostic Federated Learning \cite{mohri2019agnostic} are two approaches that we described in Section \ref{sec:background}. SCAFFOLD \cite{karimireddy2020scaffold} uses control variates (variance reduction) to correct for any distribution drift or changes in the local clients. However, the heterogeneity assumed by these approaches does not extend to spurious correlations as we have demonstrated. 

Additionally, heuristic approaches have been proposed to tackle the heterogeneity by sharing the local device data or server-side proxy data. However, in addition to placing a burden on network bandwidth, sending local data to the server as in \cite{jeong2018communication} violates the key privacy assumption of federated learning, and sending globally-shared proxy data to all clients \cite{zhao2018federated, huang2018loadaboost} requires effort to carefully generate the data. 

There has also been work on improving the communication costs of FedAvg. FedMA \cite{wang2020federated} matches the neurons of client models before averaging them, thereby reducing the communication needed. CMFL \cite{luping2019cmfl} is an approach where every client checks if its update is relevant for model improvement and reduces the number of irrelevant updates to the server, while FedAsync \cite{xie2019asynchronous} improves training time by focusing on handling stragglers. However, these approaches assume that the training and test datasets come from the same population, which may not happen in many real-world settings, thereby calling for generalizable approaches.  

Prior work on machine learning generalization have predominantly looked at computer vision challenges. 
Data augmentation techniques aim to make the model more robust by training using instances obtained from neighbouring domains hallucinated from the training domains, and thus make the network ready for these neighbouring domains. Authors in  \cite{shankar2018generalizing} augment data using instances perturbed along directions of domain change and use a second classifier to capture this. Authors in  \cite{volpi2018generalizing} apply this to single domain data, while \cite{carlucci2019domain} apply augmentation to images during training by simultaneously solving an auxiliary unsupervised jigsaw puzzle alongside.

To obtain stable predictions for unseen test distributions, a desirable model relies on only invariant features \cite{scholkopf2012causal} and causal model discovery \cite{pearl2009causality} aims to find an underlying causal graph to obtain the invariant feature set. 
Decomposition based approaches represent the parameters of the network as the sum of a common parameter and domain-specific parameters during training \cite{daume2009frustratingly}. Authors in \cite{khosla2012undoing} applied decomposition to domain generalization by retaining only the common parameter for inference and \cite{li2017deeper} extended this work to Convolutional Neural Networks (CNNs) where each layer of the network was decomposed into common and specific low-rank components. Recently, \cite{piratla2020efficient} proposed a more efficient approach that decomposes only the last layer, imposes loss on both the common and domain-specific parameters, and constrains the two parts to be orthogonal. 

Another approach is to pose the generalization problem as a meta-learning task which has been studied
either in the context of few-shot supervised learning methods which adapt using small amounts of labeled data from the new domain \cite{finn2017model}, distribution shifts in only test domains \cite{dou2019domain}, or only considering label shifts \cite{lipton2018detecting}.
There have also been some prior work on generalization for other applications like business process predictions \cite{venkateswaran2021robust} and medical diagnosis using human annotated spurious features \cite{srivastava2020robustness}.

\section{Conclusion}

We present an approach to improve the generalizability of models trained using federated learning. FedGen overcomes the brittleness of current empirical risk minimization based approaches in the presence of spurious correlations in the training data, while also not needing any prior knowledge of training distributions unlike existing generalization approaches. We develop a masking function for the local clients that allows them to collaboratively identify and distinguish between invariant and spurious features during the training process. The final aggregated model thus relies on the invariant feature set, thereby ensuring its robustness and generalizability to unseen test distributions. We prove that our masking function minimizes loss even under the most adverse test distributions. We show that FedGen achieves good generalization and outperforms existing federated learning approaches on a variety of real-world datasets. We believe that this work motivates the need to improve generalizability in federated learning and as part of future work, we will work on improving the communication costs and extend our approach to other models that satisfy the minimax optimization such as GANs, as well as other prediction tasks.

\bibliographystyle{IEEEtran}
\bibliography{./ref}

\end{document}